\documentclass[11pt]{article}

\usepackage{graphicx} %
\usepackage{subfigure} 
\usepackage{amsmath,amssymb,latexsym}
\usepackage{amsthm,algorithm,algorithmic}
\usepackage{natbib}
\usepackage{hyperref}
\usepackage{fullpage}

\def\mbf{\mathbf}
\def\mc{\mathcal}
\def\mbs{\boldsymbol}

\newcommand{\expe}[1]{ \mathbb{E}\left[ #1 \right] }
\newcommand{\E}{ \mathbb{E} }
\newcommand{\prob}{\mathbb{P}}
\newcommand{\norm}[1]{ \left\| #1 \right\| }

\def\defeq{ \stackrel{\Delta}{=} }
\newcommand{\field}[1]{\mathbb{#1}}
\newcommand{\R}{\field{R}}
\newcommand{\argmax}{\mathop{\mathrm{argmax}}}

\newcommand{\x}{\mbf{x}}
\newcommand{\Xmap}{\hat{\mbf{x}}_{\mathsf{MAP}}}
\newcommand{\Xrmap}{\hat{\mbf{x}}_{\mathsf{R-MAP}}}

\newcommand{\cX}{\mc{X}}
\newcommand{\dompot}{\mathrm{Dom}(\theta)}
\newcommand{\pot}{\theta}
\newcommand{\gumv}{\gamma}
\newcommand{\gum}[1]{\gumv(#1)}
\newcommand{\gumi}[2]{\gumv_{#1}(#2)}
\newcommand{\gumset}{\mbs{\Gamma}}

\def\Ent{\mathop{\mathrm{Ent}}\nolimits}
\def\Var{\mathop{\mathrm{Var}}\nolimits}
\newcommand{\cvxf}{Q}  %
\newcommand{\lcdens}{q}  %
\newcommand{\h}{f}  %
\newcommand{\mvh}{F} %
\newcommand{\y}{\mbf{y}}

\newtheorem{theorem}{Theorem}

\newtheorem{corollary}{Corollary}

\renewcommand{\[}{\begin{eqnarray}}
\renewcommand{\]}{\end{eqnarray}}

\title{On Measure Concentration of Random Maximum A-Posteriori Perturbations}
\author{Francesco Orabona\thanks{FO, TH, and ADS contributed equally to this paper.} \thanks{Toyota Technological Institute at Chicago, Chicago, IL, USA, \texttt{orabona@ttic.edu}.}
\quad 
Tamir Hazan\footnotemark[1] \thanks{Department of Computer Science, University of Haifa, Haifa, Israel, and Department of Electrical Engineering and Computer Science, Massachusetts Institute of Technology, Cambridge, MA, USA, \texttt{tamir.hazan@gmail.com}.}
\quad
Anand D. Sarwate\footnotemark[1] \thanks{Toyota Technological Institute at Chicago, Chicago, IL, USA, \texttt{asarwate@ttic.edu}.}
\quad
Tommi Jaakkola\thanks{Department of Electrical Engineering and Computer Science, Massachusetts Institute of Technology, Cambridge, MA, USA, \texttt{tommi@csail.mit.edu}.}}

\begin{document} 

\maketitle

\begin{abstract} 
The maximum a-posteriori (MAP) perturbation framework has emerged as a useful approach for inference and learning in high dimensional complex models.  By maximizing a randomly perturbed potential function, MAP perturbations generate unbiased samples from the Gibbs distribution.  Unfortunately, the computational cost of generating so many high-dimensional random variables can be prohibitive.  More efficient algorithms use sequential sampling strategies based on the expected value of low dimensional MAP perturbations. This paper develops new measure concentration inequalities that bound the number of samples needed to estimate such expected values. Applying the general result to MAP perturbations can yield a more efficient algorithm to approximate sampling from the Gibbs distribution.  The measure concentration result is of general interest and may be applicable to other areas involving expected estimations.\end{abstract} 

\section{Introduction}

Modern machine learning tasks in computer vision, natural language
processing, and computational biology involve inference in
high-dimensional complex models.  Examples include scene understanding
\citep{Felzenszwalb11}, parsing \citep{Koo10}, and protein design
\citep{Sontag08}. In these settings inference involves finding likely
structures that fit the data, such as objects in images, parsers in
sentences, or molecular configurations in proteins. Each structure
corresponds to an assignment of values to random variables and the
likelihood of an assignment is based on defining potential functions
that account for interactions over these variables. Given the observed
data, these likelihoods yield a \textit{posterior probability
  distribution} on assignments known as the Gibbs distribution.
Contemporary practice gives rise to posterior probabilities that
consider potential influence of the data on the variables of the model
(high signal) as well as human knowledge about the potential
interactions between these variables (high coupling).  The resulting
posterior probability landscape is often ``ragged''; in such
landscapes Markov chain Monte Carlo (MCMC) approaches to sampling from
the Gibbs distribution may become prohibitively expensive.  This is in
contrast to the success of MCMC approaches in other settings (e.g.,
\citet{Jerrum04, Huber03}) where no data term (signal) exists.

One way around the difficulties of sampling from the Gibbs distribution is to look for the \textit{maximum a posteriori probability} (MAP) structure. Substantial effort has gone into developing algorithms for recovering MAP assignments by exploiting domain-specific structural restrictions such as super-modularity \citep{Kolmogorov-pami06} or by linear programming relaxations such as cutting-planes \citep{Sontag08, Werner08}.  A drawback of MAP inference is that it returns a single assignment; in many contemporary models with complex potential functions on many variables, there are several likely structures, which makes MAP inference less appealing.  We would like to also find these other ``highly probable'' assignments.

Recent work has sought to leverage the current efficiency of MAP
solvers to build procedures to sample from the Gibbs distribution,
thereby avoiding the computational burden of MCMC methods.  These
works calculate the MAP structure of a \textit{randomly perturbed
  potential function}.  Such an approach effectively ignores the
raggedness of the landscape that hinders MCMC.  \citet{Papandreou11}
and \citet{Tarlow12} have shown that randomly perturbing the potential
of each structure with an independent random variable that follows the
Gumbel distribution and finding the MAP assignment of the perturbed
potential function provides an unbiased sample from the Gibbs
distribution. Unfortunately the total number of structures, and
consequently the total number of random perturbations, is exponential
in the structure's dimension. Alternatively, \citet{Hazan13-gibbs} use
expectation bounds on the partition function \citep{Hazan12-icml} to
build a sampler for Gibbs distribution using MAP solvers on low
dimensional perturbations which are only linear in the dimension of
the structures.

The samplers based on low dimensional perturbations involve
calculating expectations of the value of the MAP solution after
perturbations.  In this paper we give a statistical characterization
of this value.  In particular, we prove new measure concentration
inequalities that show the expected perturbed MAP value can be
estimated with high probability using only a few random samples.  This
is an important ingredient to construct an alternative to MCMC in the
data-knowledge domain that relies on MAP solvers.  The key technical
challenge comes from the fact that the perturbations are Gumbel random
variables.  Since the Gumbel distribution is continuous, the MAP value
of the perturbed potential function is unbounded and standard
approaches such as McDiarmid's inequality do not apply.  Instead, we
derive a new Poincar\'{e} inequality for the Gumbel distribution, as
well as a modified logarithmic Sobolev inequality using the approach
suggested by \citet{BobkovL:97exp}, as described in the monograph of
\citet{Ledoux:01concentration}. These results, which are of general
interest, also guarantee that the deviation of the sampled mean of
random MAP perturbations from their expectation has an exponential
decay.

\section{Problem statement}

\paragraph{Notation:} Boldface will denote tuples or vectors and calligraphic script sets.  For a tuple $\x = (x_1, x_2, \ldots, x_n)$, let $\x_{j:k} = (x_j, x_{j+1},\ldots, x_k)$. 

\subsection{The MAP perturbation framework}

Statistical inference problems involve reasoning about the states of discrete variables whose configurations (assignments of values) specify the discrete structures of interest.  Suppose that our model has $n$ variables $\x = (x_1, x_2, \ldots, x_n)$ where each $x_i$ taking values in a discrete set $\cX_i$.  Let $\cX = \cX_1 \times \cX_2 \times \cdots \times \cX_n$ so that $\x \in \cX$. Let $\dompot \subseteq \cX$ be a subset of possible configurations and $\pot : \cX \to \mathbb{R}$ be a potential function that gives a score to an assignment or structure $\x$, where $\pot(\x) = -\infty$ for $\x \notin \dompot$.  The potential function induces a probability distribution on configurations $\x$ via the Gibbs distribution: 
\begin{align}
\label{eq:gibbs}
p(\x) &\defeq \frac{1}{Z} \exp(\theta(\x)), \\
Z &\defeq \sum_{\x \in \cX} \exp(\theta(\x)). \label{eq:Z}
\end{align}
The normalization constant $Z$ is called the partition function. Sampling from \eqref{eq:gibbs} is often difficult because the sum in \eqref{eq:Z} involves an exponentially large number of terms (equal to the number of discrete structures).  In many cases, computing the partition function is in the complexity class $\#P$ (e.g., \citet{Valiant79}).

Finding the most likely assignment of values to variables is easier. As the Gibbs distribution is typically constructed given observed data, we call this the maximum a-posteriori (MAP) prediction. Maximizing \eqref{eq:gibbs}:
\[
\Xmap =  \argmax_{\x \in \cX} \pot(\x). \label{eq:M} 
\]
There are many good optimization algorithms for solving \eqref{eq:M} in cases of practical interest.
Although MAP prediction is still NP-hard in general, it is often simpler than sampling from the Gibbs distribution.
However, there are often several values of $\x$ whose scores $\pot(\x)$ are close to $\pot(\Xmap)$, and we would like to recover those as well. As an alternative to MCMC methods for sampling from the Gibbs distribution in \eqref{eq:gibbs}, we can draw samples by perturbing the potential function and solving the resulting MAP problem.  The MAP perturbation approach adds a random function $\gumv : \cX \to \mathbb{R}$ to the potential function in \eqref{eq:gibbs} and solves the resulting MAP problem:
	\begin{align}
	\Xrmap =  \argmax_{\x \in \cX} \left\{ \pot(\x) + \gum{\x} \right\}. \label{eq:Rmap} 
	\end{align}
The random function $\gum{\cdot}$ associates a random variable to each $\x \in \cX$.  The simplest approach to designing a perturbation function is to associate an independent and identically distributed (i.i.d.) random variable $\gum{\x}$ for each $\x \in \cX$.  We can find the distribution of the randomized MAP predictor in \eqref{eq:Rmap} when $\{ \gum{\x} : \x \in \cX \}$ are i.i.d.; in particular, suppose each $\gum{\x}$  a Gumbel random variable 
with zero mean, variance $\pi^2/6$, and cumulative distribution function
	\begin{align}
	G( y ) = \exp( - \exp( - (y + c))),
	\label{eq:gumbelcdf}
	\end{align}
where $c \approx 0.5772$ is the Euler-Mascheroni constant.  The following result characterizes the distribution of the randomized predictor $\Xrmap$ in \eqref{eq:Rmap}. 

\begin{theorem} \cite{Gumbel54}
\label{theorem:z}
Let $\gumset = \{ \gum{\x} : \x \in \cX \}$ be a collection of i.i.d. Gumbel random variables whose distribution is given by \eqref{eq:gumbelcdf}.  Then
\begin{align}
\label{eq:gumbel-e}
\log Z &= \E_{\gumset}\left[ \max_{\x \in \cX} 
	\left\{ \pot(\x) + \gum{\x} \right\} \right],  \\
\frac{\exp(\pot(\hat{\x}))}{Z} &= \prob_{\gumset} \left(
	\hat{\x} = \argmax_{\x \in \cX} \left\{ \pot(\x) + \gum{\x} \right\}
	\right). \nonumber
\end{align}
\end{theorem}

The max-stability of the Gumbel distribution provides a straightforward approach to generate unbiased samples from the Gibbs distribution -- simply generate the perturbations in $\gumset$ and solve the problem in \eqref{eq:Rmap}.  However, because $\gumset$ contains $|\cX|$ i.i.d. random variables, this approach to inference has complexity which is exponential in $n$.  

\subsection{Sampling from the Gibbs distribution using low dimensional perturbations}

Sampling from the Gibbs distribution is inherently tied to estimating the partition function in \eqref{eq:Z}. If we could compute $Z$ exactly, then we could sample $x_1$ with probability proportional to $\sum_{x_2, \ldots, x_n} \exp (\pot(\x))$, and for each subsequent dimension $i$, sample $x_i$ with probability proportional to $\sum_{x_{i+1},\ldots,x_n} \exp (\pot(\x))$, yielding a Gibbs sampler.  However, this involves computing the partition function, which is hard. 
Instead, \cite{Hazan13-gibbs} use the representation in \eqref{eq:gumbel-e} 
to derive a family of self-reducible upper bounds on $Z$ and then use these upper bounds in an iterative algorithm that samples from the Gibbs distribution using low dimensional random MAP perturbations.  This gives a method which has complexity linear in $n$.

In the following, instead of the $|\cX|$ independent random variables in \eqref{eq:Rmap}, we define the random function $\gum{\x}$ in \eqref{eq:Rmap} as the sum of independent random variables for each coordinate $x_i$ of $\x$: 
	\begin{align*}
	\gum{\x} = \sum_{i=1}^{n} \gumi{i}{x_i}.
	\end{align*}
This function involves generating $\sum_{i=1}^{n} |\cX_i|$ random variables for each $i$ and $x_i \in \cX_i$.  Let
	\begin{align*}
	\gumset &= \bigcup_{i=1}^{n} \left\{ \gumi{i}{x_i} : x_i \in \cX_i \right\} 
	\end{align*}
be a collection of $\sum_{i} |\cX_i|$ i.i.d. Gumbel random variables with distribution \eqref{eq:gumbelcdf}. The sampling algorithm in Algorithm \ref{alg:unbiased} uses these random perturbations to draw unbiased samples from the Gibbs distribution.  For a fixed $\x_{1:(j-1)} = (x_1, \ldots, x_{j-1})$, define
	\begin{align}
        V_j = \max_{\x_{j:n}} \left\{ \pot(\x) + \sum_{i=j}^n \gumi{i}{x_i} \right\}.
	\label{eq:v}
	\end{align}
The sampler proceeds sequentially -- for each $j$ it constructs a distribution $p_j(\cdot)$ on $\cX_j \cup \{r\}$, where $r$ indicates a ``restart''  and attempts to draw an assignment for $x_j$.  If it draws $r$ then it starts over again from $j = 1$, and if it draws an element in $\cX_j$ it fixes $x_j$ to that element and proceeds to $j+1$.
	
\begin{algorithm}[t]
\caption{Sampling with low-dimensional random MAP perturbations from the Gibbs distribution \citep{Hazan13-gibbs}}
Iterate over $j=1,...,n$, while keeping fixed $\x_{1:(j-1)}$
\begin{enumerate}
\item For each $x_j \in \cX_j$, set $p_j(x_j) = \frac{\exp(\E_{\gumset}[V_{j+1}])}{\exp(\E_{\gumset}[V_j])}$, where $V_j$ is given by \eqref{eq:v}
\item Set $p_j(r) = 1 - \sum_{x_j \in \cX_j} p(x_j)$
\item Sample an element in $\cX_j \cup \{r\}$ according to $p_j(\cdot)$. If $r$ is sampled then reject and restart with $j=1$. Otherwise, fix the sampled element $x_j$ and continue the iterations
\end{enumerate}
Output:
$\x = (x_1,...,x_n)$
\label{alg:unbiased}
\end{algorithm}

Implementing Algorithm~\ref{alg:unbiased} requires estimating the expectations $\E_{\gumset}[V_j]$ in \eqref{eq:v}. In this paper we show how to estimate $\E_{\gumset}[V_j]$ and bound the error with high probability by taking the sample mean of $M$ i.i.d. copies of $V_j$.  Specifically, we show that the estimation error decays exponentially with $M$.  To do this we derive a new measure concentration result by proving a modified logarithmic Sobolev inequality for the product of Gumbel random variables.  To do so we derive a more general result -- a Poincar\'{e} inequality for log-concave distributions that may not be log-strongly concave, i.e., for which the second derivative of the exponent is not bounded away from zero.

\subsection{Measure concentration}

We can think of the maximum value of the perturbed MAP problem as a function of the associated perturbation variables $\gumset = \{ \gumi{i}{x_i} : i \in [n], x_i \in \cX_i \}$.  There are $m \defeq |\cX_1| + |\cX_2| + \cdots + |\cX_n|$ i.i.d. random variables in $\gumset$. For practical purposes, e.g., to estimate the quality of the sampling algorithm in Algorithm~\ref{alg:unbiased}, it is important to evaluate the deviation of its sampled mean from its expectation. For notational simplicity we would only describe the deviation of the maximum value of the perturbed MAP from its expectation, namely
	\begin{align}
	\mvh(\gumset) = V_1 - \expe{V_1}.
	\label{eq:funcform}
	\end{align}	
Since the expectation is a linear function, $\expe{\mvh} = \int \mvh(\gumset) d \mu(\gumset) = 0$ is zero, with respect to any measure $\mu$ on $\gumset$. The deviation of $\mvh(\gumset)$ is dominated by its moment generating function 
	\begin{align}
	\Lambda(\lambda) \defeq \expe{ \exp(\lambda \mvh) }.
	\label{eq:mgf}
	\end{align}
That is, for every $\lambda > 0$, 
	\begin{align*}
	\prob\left( \mvh(\mbs{\gamma}) \ge r \right) \le \Lambda(\lambda) / \exp( - \lambda r).
	\end{align*}
Many measure concentration results such as McDiarmid's inequality rely on bounds on the variation of $\mvh(\gumset)$. Unfortunately, this does not hold for MAP perturbations  and instead we use the log-Sobolev approach bound \eqref{eq:mgf}.  Specifically, we want to construct a differential bound on the $\lambda-$scaled cumulant generating function:   
	\begin{align}
	H(\lambda) \defeq \frac{1}{\lambda} \log \Lambda(\lambda).
	\label{eq:lcgf}
	\end{align}
First note that that by L'H\^{o}pital's rule $H(0) = \frac{ \Lambda'(0) }{ \Lambda(0) } = \int \mvh d\mu^n = 0$, so we may represent $H(\lambda)$ by integrating its derivative: $H(\lambda) = \int_0^\lambda H'(\hat \lambda) d \hat \lambda$. Thus to bound the moment generating function it suffices to bound $H'(\lambda) \le \alpha(\lambda)$ for some function $\alpha(\lambda)$. A direct computation of $H'(\lambda)$ translates this bound to    
	\begin{align}
	\lambda \Lambda'(\lambda) - \Lambda(\lambda) \log \Lambda(\lambda) \le \lambda^2 \Lambda(\lambda) \alpha(\lambda).
	\label{eq:logsobdiff}
	\end{align}
The left side of \eqref{eq:logsobdiff} turns out to be the so-called functional entropy \cite{Ledoux:01concentration} of the function $h = \exp(\lambda F)$ with respect to a measure $\mu$:
	\begin{align*}
	\Ent_\mu(h) &\defeq \int h \log h d\mu - \left( \int h d\mu \right) \log \int h d\mu.
	\end{align*}
Unlike McDiarmid's inequality, this approach provides measure concentration for unbounded functions, such those arising from MAP perturbations.  

A log-Sobolev inequality upper-bounds the entropy $\Ent_\mu(h)$ in terms of an integral involving $\norm{ \nabla \mvh }^2$.  They are appealing to derive measure concentration results in product spaces, i.e., for functions of subsets of variables $\gumset$, because it is sufficient to prove a log-Sobolev inequality on a single variable function $\h$.  Given such a scalar result, the additivity property of the entropy (e.g., \citep{Boucheron04}) extends the inequality to functions $\mvh$ of many variables.  In this work we derive a log-Sobolev inequality for the Gumbel distribution, by bounding the variance of a function by its derivative:       
	\begin{equation}
        \begin{split}
	\label{eq:poincare}
	\Var_{\mu}(f) \defeq \int f^2 d \mu - \left(\int f d \mu \right)^2 \le C \int | f'|^2 d \mu.
	\end{split}
        \end{equation}
This is called a Poincar\'{e} inequality, proven originally for the Gaussian case. We prove such an inequality for the Gumbel distribution, which then implies the log-Sobolev inequality and hence measure concentration.  We then apply the result to the MAP perturbation framework.

\subsection{Related work} 

We are interested in efficient sampling from the Gibbs distribution in \eqref{eq:gibbs} when $n$ is large an the model is complex due to the amount of data and the domain-specific modeling.  This is often done with MCMC (cf. \citet{Koller09}), which may be challenging in ragged probability landscapes. MAP perturbations use efficient MAP solvers as black box, but the statistical properties of the solutions, beyond Theorem \ref{theorem:z}, are still being studied.  \citet{Papandreou11} consider probability models that are defined by the maximal argument of randomly perturbed potential function, while \citet{Tarlow12} considers sampling techniques for such models and \citet{Keshet11} explores the generalization bounds for such models. Rather than focus on the statistics of the solution (the $\argmax_{\x}$) we study statistical properties of the MAP value (the $\max_{\x}$) of the estimate in \eqref{eq:Rmap}.

\citet{Hazan12-icml} used the random MAP perturbation framework to derive upper bounds on the partition function in \eqref{eq:Z}, and \citet{Hazan13-gibbs} derived the unbiased sampler in Algorithm~\ref{alg:unbiased}.  Both of these approaches involve computing an expectation over the distribution of the MAP perturbation, which can be estimated by sample averages.  This paper derives new measure concentration results that bound the error of this estimate in terms of the number of samples, making Algorithm~\ref{alg:unbiased} practical.

Measure concentration has appeared in many machine learning analyses, most commonly to bound the rate of convergence for risk minimization, either via empirical risk minimization (ERM) (e.g., \citet{Bartlett03}) or in PAC-Bayesian approaches (e.g., \citet{McAllester03}).  In these applications the function for which we want to show concentration is ``well-behaved'' in the sense that the underlying random variables are bounded or the function satisfies some bounded-difference or self-bounded conditions conditions, so measure concentration follows from inequalities such as \citet{Bernstein46}, Azuma-Hoeffding~\citep{Azuma67,Hoeffding63,McDiarmid89}, or \citet{Bousquet03}.  However, in our setting, the Gumbel random variables are not bounded, and random perturbations may result in unbounded changes of the perturbed MAP value. 

There are several results on measure concentration for Lipschitz functions of \textit{Gaussian} random variables (c.f. Maurey and \citet{Pisier85}).   In this work we use logarithmic Sobolev inequalities~\cite{Ledoux:01concentration} and prove a new measure concentration result for \textit{Gumbel} random variables.  To do this we generalize a classic result of \citet{BrascampL76} on Poincar\'{e} inequalities to non-strongly log-concave distributions, and also recover the concentration result of \citet{BobkovL:97exp} for functions of Laplace random variables.

\section{Concentration of measure} 

In this section we prove the main technical results of this paper -- a new Poincar\'{e} inequality for log concave distributions and the corresponding measure concentration result. We will then specialize our result to the Gumbel distribution and apply it to the MAP perturbation framework.
Because of the tensorization property of the functional entropy, it is sufficient for our case to prove an inequality like \eqref{eq:poincare} for functions $\h$ of a single random variable with measure $\mu$.

\subsection{A Poincar\'{e} inequality for log-concave distributions}

Our Theorem \ref{theo:t1} in this section generalizes a celebrated result of \citet[Theorem~4.1]{BrascampL76} to a wider family of log-concave distributions and strictly improves their result. 
For an appropriately scaled convex function $\cvxf$ on $\R$, the function $\lcdens(y)=\exp(-\cvxf(y))$ defines a density on $\R$ corresponding to a log concave measure $\mu$.  Unfortunately, their result is restricted to distributions for which $\cvxf(y)$ is strongly convex.  The Gumbel distribution with CDF \eqref{eq:gumbelcdf} has density
	\begin{align}
	g(y) = \exp\left( - \left( y + c + \exp( - (y + c) ) \right) \right),
	\label{eq:gumbelpdf}
	\end{align}
and the second derivative of $y + c + \exp( - (y + c) )$ cannot be lower bounded by any constant greater than $0$, so it is not log-strongly convex.  

\begin{theorem}
\label{theo:t1}
Let $\mu$ be a log-concave measure with density $\lcdens(y)=\exp(-\cvxf(y))$, where $\cvxf: \R \rightarrow \R$ is convex function satisfying the following conditions:
\begin{itemize}
\item $\cvxf$ has a unique minimum in a point $y=a$
\item $\cvxf$ is twice continuously differentiable in each point of his domain, except possibly in $y=a$
\item $\cvxf'(y)\neq 0$ for any $y\neq a$
\item $\lim_{y \rightarrow a^\pm} \cvxf'(y)\neq0$ or $\lim_{y\rightarrow a^\pm} \cvxf''(y)\neq0$
\end{itemize}
Let $\h : \R \rightarrow \R$ a continuous function, differentiable almost everywhere, such that
\begin{equation}
\label{eq:reg_condition_h}
\lim_{y \rightarrow \pm \infty} \  \h(y) \lcdens(y) =0,
\end{equation}
then for any $0 \leq \eta <1$, such that $\frac{\cvxf''(y)}{|\cvxf'(y)|}+\eta |\cvxf'(y)| \neq 0, \ \forall y \in \R\setminus \{a\}$, we have
\begin{equation*}
\Var_{\mu}(\h) \leq \frac{1}{1-\eta}\int_{\R} \frac{(\h'(y))^2}{\cvxf''(y)+ \eta (\cvxf'(y))^2} \lcdens(y) dy.
\end{equation*}
\end{theorem}
\begin{proof}
The proof is based on the one in \citet{BrascampL76}, but it uses a different strategy in the final critical steps.
We first observe that for any $K \in \R$,
\begin{equation}
\label{eq:var_const_bound}
\Var_{\mu}(\h) \leq \int_{\R} (\h(y)-K)^2 d\mu,
\end{equation}
so we will focus on bounding the left-hand side of \eqref{eq:var_const_bound} for the particular choice of $K=h(a)$.

Let $\tilde{\h}(y) \defeq \h(y)-\h(a)$ and $U(y) \defeq \frac{\tilde{\h}(y)^2 \lcdens(y)}{\cvxf'(y)}$.  Note that $d\mu = q(y) dy$.
We have that 
\begin{align*}
U'(y)= \frac{2 \tilde{\h}'(y) \tilde{\h}(y) \lcdens(y)}{\cvxf'(y)} - \tilde{\h}(y)^2 \lcdens(y) \left( \frac{\cvxf''(y)}{(\cvxf'(y))^2} +1\right).
\end{align*}
Rearranging terms and integrating, we see that 	
\begin{align*}
\int \tilde{\h}(y)^2 \lcdens(y) dy = \int \left(\frac{2 \tilde{\h}'(y) \tilde{\h}(y) }{\cvxf'(y)} - \frac{\tilde{\h}(y)^2 \cvxf''(y)}{(\cvxf'(y))^2} \right) \lcdens(y) dy -U(y).
\end{align*}

We now consider the integral between $-\infty$ and $a$ (analogous reasoning holds for the one between $a$ and $+\infty$).
We claim that $\lim_{y \to a^-} U(y)=0$. There are two possible cases: $\cvxf'(a)\neq 0$ and $\cvxf'(a) = 0$. In the first case the claim is obvious, in the second case we have $\lim_{y \to a^-} \frac{\tilde{\h}(y)^2}{\cvxf'(y)} = \lim_{y \to a^-} \frac{2 \h'(y) \tilde{\h}(y)}{\cvxf''(y)}=0$, and anagously for the limit from the left.
Using \eqref{eq:reg_condition_h} too, we have
\begin{align*}
\int_{-\infty}^{a} \tilde{\h}(y)^2 \lcdens(y) dy &= \lim_{\epsilon \rightarrow 0^-} 
	\int_{-\infty}^{a+\epsilon} 
	\left(\frac{2 \tilde{\h}'(y) \tilde{\h}(y)}{\cvxf'(y)} 
		-  \frac{\tilde{\h}(y)^2 \cvxf''(y)}{(\cvxf'(y))^2} \right) \lcdens(y) dy\\
&\leq \lim_{\epsilon \rightarrow 0^-} 
	\int_{-\infty}^{a+\epsilon} 
	\left(\frac{2 |\tilde{\h}'(y)| |\tilde{\h}(y)|}{|\cvxf'(y)|} 
		- \frac{\tilde{\h}(y)^2 \cvxf''(y)}{(\cvxf'(y))^2} \right) \lcdens(y) dy\\
&\leq \lim_{\epsilon \rightarrow 0^-} 
	\int_{-\infty}^{a+\epsilon} 
	\left( \frac{\tilde{\h}'(y)^2}{\cvxf''(y) 
		+ \eta (\cvxf'(y))^2} +  \eta \tilde{\h}(y)^2 \right) \lcdens(y) dy,
\end{align*}
where in the second inequality we used $2 \alpha \beta \leq \frac{\alpha^2}{\zeta} + \beta^2 \zeta$, for any $\alpha,\zeta \in \R$ and $\zeta>0$, with $\alpha=|\tilde{\h}'(y)|$, $\beta=|\tilde{\h}(x)|$, and $\zeta=\frac{\cvxf''(y)}{|\cvxf'(y)|}+\eta |\cvxf'(y)|$.
Reasoning in the same way for the interval $[a,+\infty)$, reordering the terms, and using \eqref{eq:var_const_bound}, we have the result.
\end{proof}

The main difference between Theorem \ref{theo:t1} and the result of \citet[Theorem~4.1]{BrascampL76} is that the latter requires the function $\cvxf$ to be strongly convex.  Our result holds for non-strongly concave functions including the Laplace and Gumbel distributions.
If we take $\eta = 0$ in Theorem \ref{theo:t1} we recover the original result of \citet[Theorem~4.1]{BrascampL76}.  For the case $\eta = 1/2$, Theorem \ref{theo:t1} yields the Poincar\'{e} inequality for the Laplace distribution given in \citet{Ledoux:01concentration}.  Like the Gumbel distribution, the Laplace distribution is not strongly log-concave and previously required an alternative technique to prove measure concentration \citet{Ledoux:01concentration}.   The following gives a Poincar\'{e} inequality for the Gumbel distribution.

\begin{corollary}
\label{cor:cor1}
Let $\mu$ be the measure corresponding to the Gumbel distribution and $\lcdens(y) = g(y)$ in \eqref{eq:gumbelpdf}.  For any function $\h$ that satisfies the conditions in Theorem~\ref{theo:t1}, we have
	\begin{align}
	\Var_{\mu}(\h) \leq 4 \int_{\R} (\h'(y))^2 d\mu.
	\label{eq:gumbel:poincare}
	\end{align}
\end{corollary}

\begin{proof}
For the Gumbel distribution we have $\cvxf(y) = y + c + \exp( - (y + c) )$ in Theorem \ref{theo:t1}, so
	\begin{align*}
	\cvxf''(y) + \eta( \cvxf'(y) )^2 = e^{-(y+c)} + \eta ( 1 - e^{-(y+c)} )^2.
	\end{align*}	
We want an lower bound for all $y$.  Minimizing,
	\begin{align*}
	e^{-(y+c)} = 2 \eta (1 - e^{-(y+c)}) e^{-(y+c)}
	\end{align*}
or $e^{-(y+c)} = 1 - \frac{1}{2 \eta}$, so the lower bound is $1 - \frac{1}{2 \eta} + \frac{1}{4 \eta}$ or $\frac{4 \eta - 1}{4 \eta}$ for $\eta > \frac{1}{2}$.  For $\eta \le \frac{1}{2}$,
	\begin{align*}
	\eta + (1 - 2 \eta) e^{-(y+c)} + e^{-2(y+c)} \ge \eta.
	\end{align*}
So $\min\left\{ \frac{4 \eta}{(4 \eta - 1)(1 - \eta)}, \frac{1}{\eta (1 - \eta)} \right\} = 4$ at $\eta = \frac{1}{2}$, so applying Theorem \ref{theo:t1} we obtain \eqref{eq:gumbel:poincare}.
\end{proof}

\subsection{Measure concentration for the Gumbel distribution}

In the MAP perturbations such as that in \eqref{eq:v}, we have a function of many random variables.  We now derive a result based on the Corollary~\ref{cor:cor1} to bound the moment generating function for random variables defined as a function of $m$ random variables.  This gives a measure concentration inequality for the product measure $\mu^m$ of $\mu$ on $\R^m$, where $\mu$ corresponds to a scalar Gumbel random variable.

\begin{theorem}
\label{theorem:exp_moment}
Let $\mu$ denote the Gumbel measure on $\R$ and let $\mvh: \R^m \rightarrow \R$ be a function such that $\mu^m$-almost everywhere we have $\norm{ \nabla \mvh }^2 \le a^2$ and $\norm{ \nabla \mvh }_{\infty} \le b$.  Furthermore, suppose that for $\y = (y_1, \ldots, y_m)$,
\begin{align*}
\lim_{y_i \rightarrow \pm \infty} 
	\mvh(y_1, \ldots, y_m) 
	\prod_{i=1}^{m} g(y_i) = 0,
\end{align*}
where $g(\cdot)$ is given by \eqref{eq:gumbelpdf}.
Then, for any $r \geq 0$ and any $|\lambda| \leq \frac{1}{10 b}$, we have
	\begin{align*}
	\E[\exp(\lambda (\mvh-\E[\mvh]))] \leq \exp(5 a^2 \lambda^2).
	\end{align*}
\end{theorem}

\begin{proof}
For each $i=1,2\ldots,m$, we can think of $\mvh$ as a scalar function $\h_i$ of its $i$-th argument for $i=1,\ldots,m$. Using Theorem~5.14 of \citet{Ledoux:01concentration} and Corollary~\ref{cor:cor1}, for any $|\lambda| b \leq \rho \leq 1$, 
	\begin{align*}
	\Ent_{\mu_i} (\exp(\lambda \h_i)) \leq 2 \lambda^2 \left(\frac{1+ \rho}{1- \rho}\right)^2 \exp(2 \sqrt{5} \rho) \int |	\partial_i \mvh|^2 d \mu_i.
	\end{align*}
We now use Proposition~5.13 in \citet{Ledoux:01concentration} to tensorize the entropy by summing over $i=1$ to $m$:
	\begin{align*}
	\Ent_{\mu^m} (\exp(\lambda \h_i)) &\leq 2 \lambda^2 \left(\frac{1+ \rho}{1- \rho}\right)^2 \exp(2 \sqrt{5} \rho) \int \sum_{i=1}^m |\partial_i \mvh|^2 \exp(\lambda \mvh) d\mu^m \\
	&\leq 2 \lambda^2 \left(\frac{1+ \rho}{1- \rho}\right)^2 \exp(2 \sqrt{5} \rho) a^2 \int \exp(\lambda \h) d\mu^m.
	\end{align*}
Hence, choosing $\rho = \frac{1}{10}$, we obtain, for any $|\lambda| \leq \frac{1}{10 b}$
	\begin{align}
	\label{eq:the2_eq1}
	\Ent_{\mu^m} (\exp(\lambda \mvh)) \leq 5  a^2 \lambda^2  \E_{\mu^m}[\exp(\lambda \mvh)].
	\end{align}
Recall the moment generating function in \eqref{eq:mgf} and $\lambda-$scaled cumulant generating function in \eqref{eq:lcgf}, and note that $H(0)=\E[\mvh]$.  We now use Herbst's argument~\cite{Ledoux:01concentration}.  Using \eqref{eq:the2_eq1} we have
	\begin{align}
	H'(\lambda)=\frac{\Ent_{\mu^m} (\exp(\lambda \mvh))}{\lambda^2 \Lambda(\lambda)} \leq 5  a^2.
	\label{eq:kderiv}
	\end{align}
Integrating \eqref{eq:kderiv} we get 
	\begin{equation*}
	H(\lambda)\leq H(0)+ 5  a^2 \lambda= \E[\mvh] + 5 a^2 \lambda,
	\end{equation*}
Now, from the definition of $H(\lambda)$, this implies
	\begin{equation*}
	\log \E[\exp(\lambda \mvh)] \leq \lambda \E[\mvh] + 5  a^2 \lambda^2~. \qedhere
	\end{equation*}
\end{proof}

With this lemma we can now upper bound the error in estimating the average $\E[\mvh]$ of a function $\mvh$ of $m$ i.i.d. Gumbel random variables by generating $M$ independent samples of $\mvh$ and taking the sample mean.

\begin{corollary}
\label{cor:cor2}
Consider the same assumptions of Theorem~\ref{theorem:exp_moment}.  Let $\eta_1, \eta_2, \ldots, \eta_M$ be $M$ i.i.d. random variables with the same distribution as $\mvh$.  Then with probability at least $1-\delta$,
	\begin{equation*}
	\frac{1}{M} \sum_{j=1}^{M} \eta_j - \E[\mvh]
		\le 
		\max\left(\frac{20 b}{M} \log\frac{1}{\delta},\sqrt{\frac{20 a^2}{M} \log\frac{1}{\delta}}\right).
	\end{equation*}
	\end{corollary}

\begin{proof}
From the independence assumption, using the Markov inequality, we have that
\begin{align*}
\prob\left( \sum_{j=1}^M \eta_j \leq M \E[\mvh] + M r\right) \leq \exp(-M \E[\mvh]-M r) \prod_{j=1}^M \E[\exp(\lambda \eta_j)].
\end{align*}
Applying Theorem~\ref{theorem:exp_moment}, we have, for any $|\lambda| \leq \frac{1}{10 b}$,
\begin{align*}
&\prob\left( \frac{1}{M} \sum_{j=1}^M \eta_j \leq \E[\mvh] + r\right) \le \exp(M (5 a^2 \lambda^2 -\lambda r)).
\end{align*}
Optimizing over $\lambda$ subject to $|\lambda| \leq \frac{1}{10 b}$ we obtain
\begin{align*}
\exp(M (5 a^2 \lambda^2 -\lambda r))\leq \exp\left(-\frac{M}{20}\min\left(\frac{r}{b}, \frac{r^2}{a^2}\right)\right).
\end{align*}
Equating the left side of the last inequality to $\delta$ and solving for $r$, we have the stated bound.
\end{proof}

\subsection{Application to MAP perturbations}

To apply these results to the MAP perturbation problem we must calculate the parameters in the bound given by the Corollary~\ref{cor:cor2}. Let $\mvh(\gumset)$ be the random MAP perturbation as defined  in \eqref{eq:funcform}. This is a function of $m \defeq \sum_{i=1}^{n} |\cX_i|$ i.i.d. Gumbel random variables. The (sub)gradient of this function is structured and points toward the $\gumi{i}{x_i}$ that relate to the maximizing assignment in $\Xrmap$ defined in \eqref{eq:Rmap}, when $\gum{\x} = \sum_{i=1}^{n} \gumi{i}{x_i}$, that is
\begin{align*}
\frac{\partial \mvh(\gumset)}{\partial \gumi{i}{x_i}}  = \Bigg\{ \begin{array}{ll} 1 & \mbox{if} \;\; x_i \in \Xrmap \\ 0 & \mbox{otherwise}. \end{array} 
\end{align*}
Thus the gradient satisfies $\norm{ \nabla \mvh }^2 = n$ and $\norm{ \nabla \mvh }_{\infty} = 1$ almost everywhere, so $a^2=n$ and $b=1$. Suppose we sample $M$ i.i.d. copies $\gumset_1, \gumset_2, \ldots, \gumset_M$ copies of $\gumset$ and estimate the deviation from the expectation by $\frac{1}{M} \sum_{i=1}^{M} \mvh(\gumset_i)$. We can apply Corollary \ref{cor:cor2} to both $\mvh$ and $-\mvh$ to get the following double-sided bound with probability $1 - \delta$:
\begin{align*}
\left| \frac{1}{M} \sum_{i=1}^{M} \mvh (\gumset_i) \right| \leq \max\left(\frac{20}{M}\log\frac{2}{\delta},\sqrt{\frac{20 n}{M} \log\frac{2}{\delta} }\right).
\end{align*}
Thus this result gives an estimation for the MAP perturbations $\expe{ \max_{\x} \left\{ \pot(\x) + \sum_{i=1}^{n} \gumi{i}{x_i} \right\}}$ that hold in high probability. %

This result can also be applied to estimate the quality of Algorithm \ref{alg:unbiased} that samples from the Gibbs distribution using MAP solvers. Now we let $\mvh$ equal $V_j$ from \eqref{eq:v}.  This is a function of $m_j \defeq \sum_{i=j}^{n} |\cX_i|$ i.i.d. Gumbel random variables whose gradient satisfies $\norm{ \nabla V_j }^2 = n-j+1$ and $\norm{ \nabla V_j }_{\infty} = 1$ almost everywhere, so $a^2 = n-j+1$ and $b = 1$.  Suppose $U = V_j- \expe{V_j}$ is a random variable that measures the deviation of $V_j$ from its expectation, and assume we sample $M_j$ i.i.d. random variable $U_1, U_2, \ldots, U_{M_j}$. We then estimate this deviation by the sample mean $\frac{1}{M_j} \sum_{i=1}^{M_j} U_i$.  Applying Corollary \ref{cor:cor2} to both $V_j$ and $-V_j$ to get the following bound with probability $1 - \delta$:
\begin{align}
\left| \frac{1}{M_j} \sum_{i=1}^{M_j} U_i \right| \leq \max\left(\frac{20}{M_j}  \log\frac{2}{\delta},\sqrt{\frac{20 (n-j+1)}{M_j}  \log\frac{2}{\delta}}\right).
\label{eq:errVj}
\end{align}
For each $j$ in Algorithm \ref{alg:unbiased}, we must estimate $|\cX_j|$ expectations $\E_{\gumset}[V_{j+1}]$, for a total at most $m$ expectation estimates.  For any $\epsilon > 0$ we can choose $\{M_j : j=1,\ldots,n\}$ so that the right side of \eqref{eq:errVj} is at most $\epsilon$ for each $j$ with probability $1 - n \delta$.  Let $\hat{p}_j(x_j)$ be the ratio estimated in the first step of Algorithm \ref{alg:unbiased}, and $\delta' = n \delta$.  Then with probability $1 - \delta'$, for all $j=1,2,\ldots,n$,
$\frac{\exp(\expe{V_{j+1}} - \epsilon)}{\exp(\expe{V_j} + \epsilon)} \le  \hat{p}_j(x_j) \le \frac{\exp(\expe{V_{j+1}} + \epsilon)}{\exp(\expe{V_j} - \epsilon)}$, or
	\begin{align*}
	\exp(-2 \epsilon) \le \frac{\hat{p}_j(x_j)}{p_j(x_j)} \le \exp(2 \epsilon).
	\end{align*}

\section{Experiments}
\label{sec:experiments}

\begin{figure}[t]
\centering
\includegraphics[width=6cm]{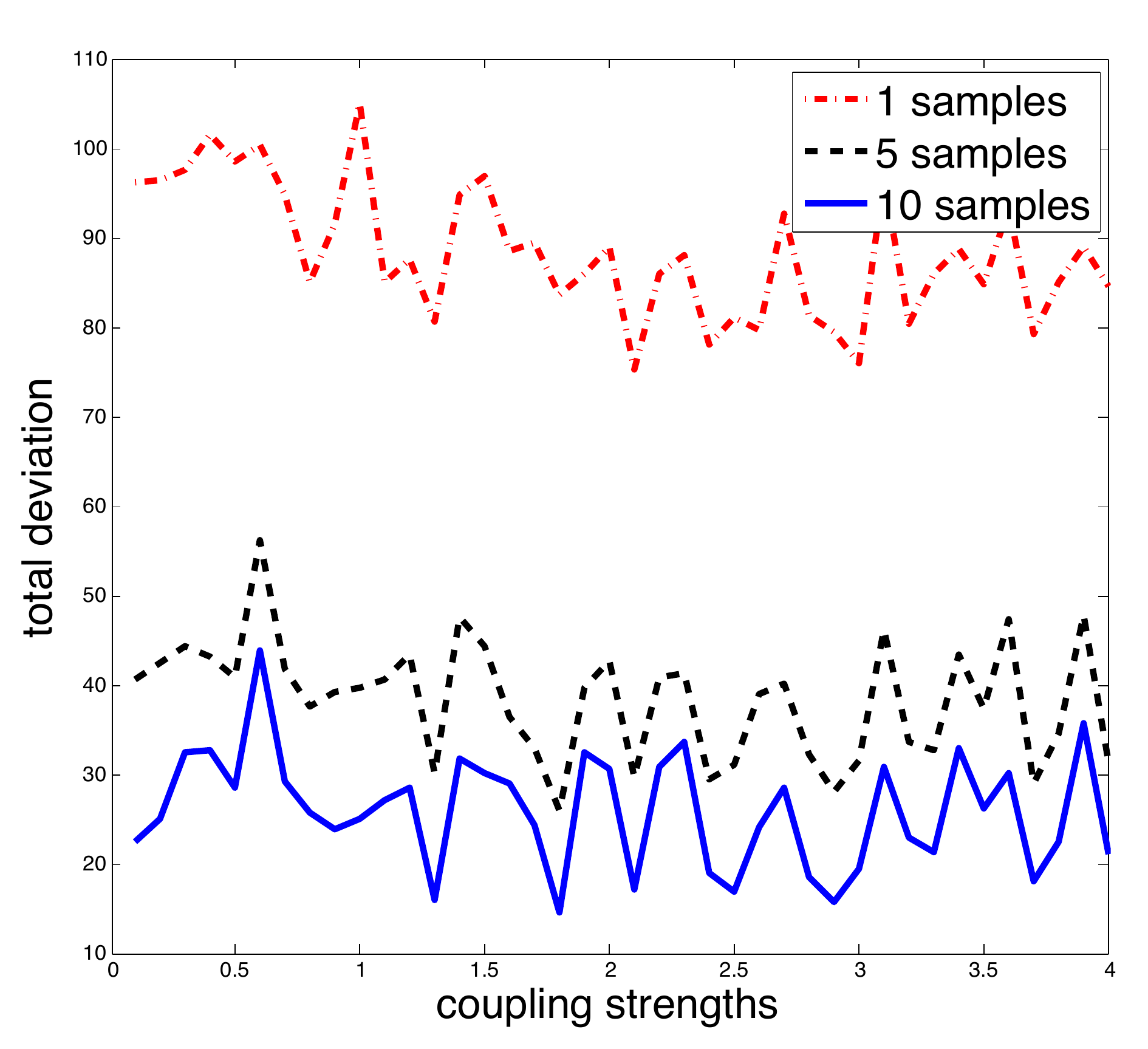}
\caption{ Error of the sample mean versus coupling strength. With only $10$ samples one can estimate the expectation well. \label{fig:samplemean}}
\end{figure}

\begin{figure*}[t]
\centering
\subfigure{
\includegraphics[width=5.5cm]{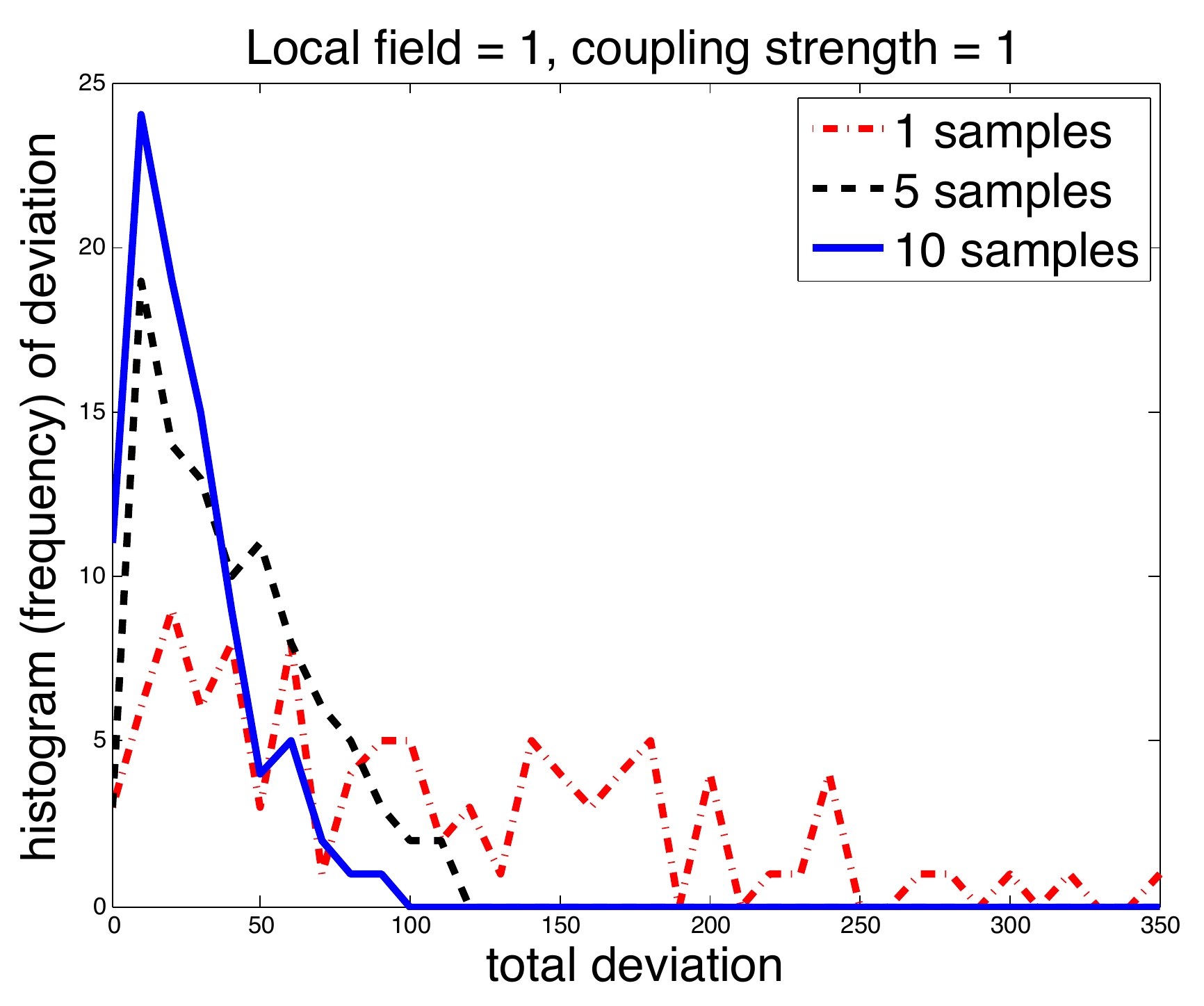}
}%
\subfigure{
\includegraphics[width=5.5cm]{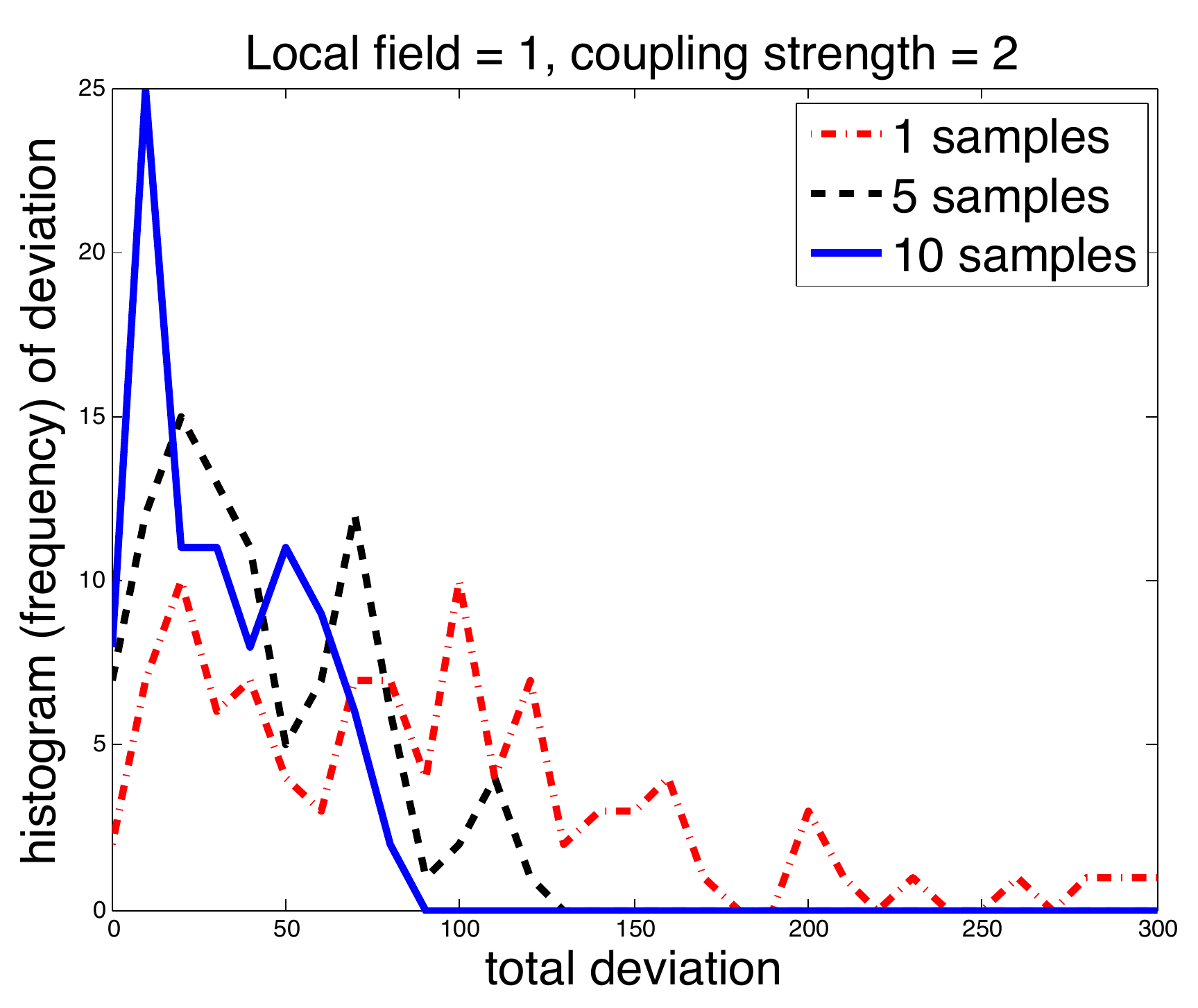}
}%
\subfigure{
\includegraphics[width=5.5cm]{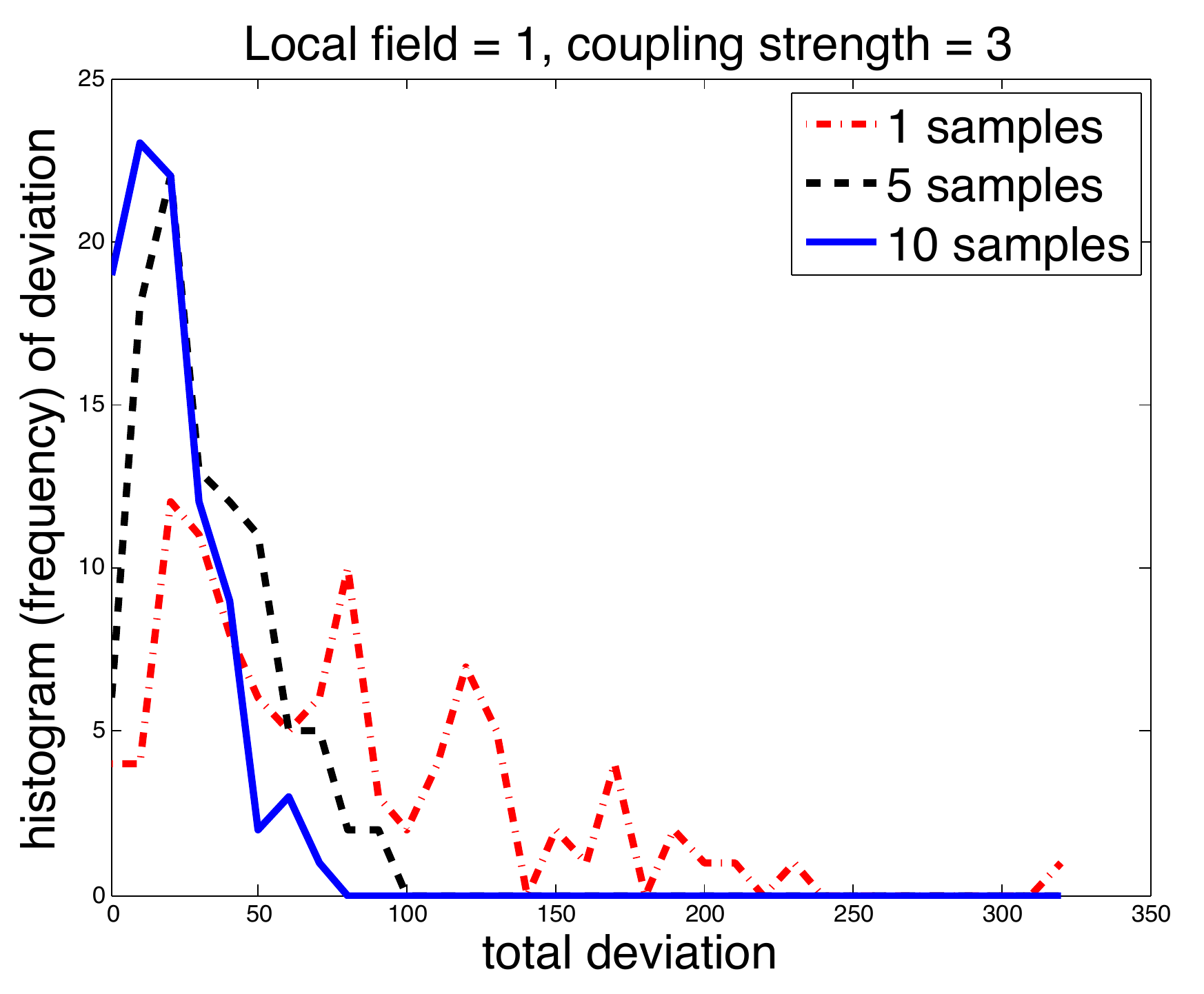}
}
\caption{ 
Histogram of MAP values for the $100 \times 100$ spin glass model. \label{fig:spin-glass}}
\end{figure*}

We evaluated our approach on a $100 \times 100$ spin glass model with $n = 10^4$ variables, for which
	\begin{align*}
	\theta(x_1,...,x_n) = \sum_{i \in V} \theta_i(x_i) + \sum_{(i,j) \in E} \theta_{i,j}(x_i,x_j)~.
	\end{align*}
where $x_i \in \{-1,1\}$. Each spin has a local field parameter $\theta_i(x_i) = \theta_i x_i$ and interacts in a grid shaped graphical structure with couplings $\theta_{i,j}(x_i,x_j) = \theta_{i,j} x_i x_j$. Whenever the coupling parameters are positive the model is called attractive since adjacent variables give higher values to positively correlated configurations.  We used low dimensional random perturbations $\gum{\x} = \sum_{i=1}^{n} \gumi{i}{x_i}$. %

The local field parameters $\theta_i$ were drawn uniformly at random from $[-1,1]$ to reflect high signal. The parameters $\theta_{i,j}$ were drawn uniformly from $[0,c]$, where $c \in [0,4]$ to reflect weak, medium and strong coupling potentials. As these spin glass models are attractive, we are able to use the graph-cuts algorithm (\citet{Kolmogorov-pami06}) to compute  the MAP perturbations efficiently. Throughout our experiments we evaluated the expected value of $\mvh(\gumset)$ with $100$ different samples of $\gumset$. We note that  we have two random variables $\gumi{i}{x_i}$ for each of the spins in the $100 \times 100$ model, thus $\gumset$ consists of $m = 2*10^4$ random variables. 

Figure \ref{fig:samplemean} shows the error in the sample mean $\frac{1}{M}\sum_{k=1}^M F(\gumset_k)$ versus the coupling strength for three different sample sizes $M=1,5,10$.  The error reduces rapidly as $M$ increases; only $10$ samples are needed to estimate the expected value of a random MAP perturbation with $10^4$ variables. To test our measure concentration result, that ensures exponential decay, we measure the deviation of the sample mean from its expectation by using $M=1,5,10$ samples. Figure \ref{fig:spin-glass} shows the histogram of the sample mean, i.e., the number of times that the sample mean has error more than $r$ from the true mean. One can see that the decay is indeed exponential for every $M$, and that for larger $M$ the decay is much faster. These show that by understanding the measure concentration properties of MAP perturbations, we can efficiently estimate the mean with high probability, even in very high dimensional spin-glass models. 

\section{Conclusion}
Sampling from the Gibbs distribution is important because it helps find near-maxima in posterior probability landscapes that are typically encountered in the high dimensional complex models.  These landscapes are often ragged due to domain-specific modeling (coupling) and the influence of data (signal), making MCMC challenging.  In contrast, sampling based on MAP perturbations ignores the ragged landscape as it directly targets the most plausible structures.  In this paper we characterized the statistics of MAP perturbations.  

To apply the low-dimensional MAP perturbation technique in practice, we must estimate the expected value of the quantities $V_j$ under the perturbations.  We derived high-probability estimates of these expectations that allow estimation with arbitrary precision.  To do so we proved more general results on measure concentration for functions of Gumbel random variables and a Poincar\'{e} inequality for non-strongly log-concave distributions.  These results hold in generality and may be of use in other applications.

The results here can be taken in a number of different directions. MAP perturbation models are related PAC-Bayesian generalization bounds, so it may be possible to derive PAC-Bayesian bounds for unbounded loss functions using our tools. Such loss functions may exclude certain configurations and are already used implicitly in computer vision applications such as interactive segmentations. More generally, Poincar\'{e} inequalities relate the variance of a function and its derivatives.  Our result may suggest new stochastic gradient methods that control variance via controlling gradients.
This connection between variance and gradients may be useful in the analysis of other learning algorithms and applications.

\bibliography{logsob-icml}
\bibliographystyle{plainnat}

\end{document}